\documentclass[10pt,twocolumn,letterpaper]{article}

\usepackage{cvpr}              % To produce the CAMERA-READY version
\usepackage{caption}
\usepackage{subcaption}

\usepackage[table,xcdraw]{xcolor}

\usepackage{multirow}
\usepackage{algpseudocode}

\usepackage[linesnumbered,ruled,vlined]{algorithm2e}

\usepackage{tikz}
\usetikzlibrary{arrows.meta, positioning, shapes.geometric}

\definecolor{cvprblue}{rgb}{0.21,0.49,0.74}
\usepackage[pagebackref,breaklinks,colorlinks,allcolors=cvprblue]{hyperref}

\def\paperID{22786} % *** Enter the Paper ID here
\def\confName{CVPR}
\def\confYear{2026}

\title{Onboarding Without Forgetting: Hypernetwork Personalization \\ with Data-Free Replay for Personalized Federated Learning}

\author{Thinh T. H. Nguyen\textsuperscript{\rm 1} \quad Le Huy Khiem\textsuperscript{\rm 2} \quad Van-Tuan Tran\textsuperscript{\rm 3} \\ Khoa D Doan\textsuperscript{\rm 1} \quad Nitesh V Chawla\textsuperscript{\rm 2} \quad Kok-Seng Wong\textsuperscript{\rm 1}\protect\thanks{Corresponding author: \url{wong.ks@vinuni.edu.vn}}\\
\textsuperscript{\rm 1}VinUniversity, Hanoi, Vietnam\\
\textsuperscript{\rm 2}University of Notre Dame, Indiana, USA\\
\textsuperscript{\rm 3}Trinity College Dublin, Dublin, Dublin, Ireland\\
{\tt\small thinh.nth@vinuni.edu.vn, kle3@nd.edu, tranva@tcd.ie}
\\
{\tt\small khoa.dd@vinuni.edu.vn, nchawla@nd.edu, wong.ks@vinuni.edu.vn}
}

\usepackage{amsthm}
\newtheorem{assumption}{Assumption}
\newtheorem{theorem}{Theorem}

\usepackage{booktabs}
\usepackage{siunitx}
\begin{document}
\maketitle

\begin{abstract}
   Federated Learning (FL) enables collaborative training across distributed clients without sharing raw data, offering strong privacy benefits. However, most methods assume all clients remain available throughout training, which is unrealistic as new clients often join over time. We study this setting, where the task and label space stay fixed but clients arrive in batches. Our analysis reveals two key challenges: updating the shared model only with new clients harms existing clients, while freezing it protects them but blocks gains from new knowledge. To capture these trade-offs, we introduce Proactive Adaptation (PA) for onboarding gains and Retroactive Improvement (RI) for changes in earlier clients without retraining. We then propose pFedDSH, which combines a central hypernetwork for personalized initialization, batch-specific binary masks for capacity preservation and allocation, and server-side data-free replay to propagate improvements without exposing client data. Experiments show that pFedDSH preserves stability for existing clients while keeping communication and adaptation costs unchanged for new clients.
\end{abstract}

\section{Introduction}
\label{sec:intro}

Personalized federated learning (pFL) handles client heterogeneity without centralizing data~\cite{kairouz2021advances, mcmahan2017communication}, but most methods assume a fixed amount of clients throughout training. This rarely holds in practice since data are partitioned across decentralized entities whose participation changes over time. In healthcare, hospitals maintain private image repositories for the same diagnostic task (e.g., pneumonia classification). New sites begin contributing months after the initial rollout, and existing sites periodically pause participation for maintenance or policy review. Both the label space and task remain constant, but the set of participating clients evolves, creating a dynamic onboarding process that standard pFL does not explicitly model~\cite{nazir2023federated}. 

\begin{figure*}[t]
\centering
\includegraphics[width=\textwidth]{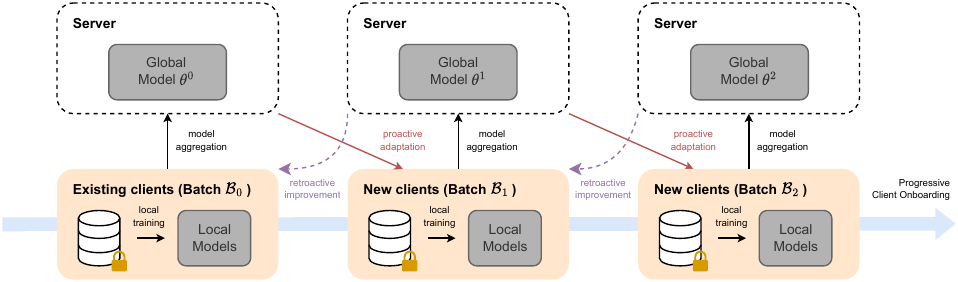}

\caption{Overview of PCO-FL. Existing clients ($\mathcal{B}_0$) retain stable performance without retraining. New clients ($\mathcal{B}_1, \mathcal{B}_2, \dots$) leverage existing knowledge and, in turn, can enhance the performance of existing clients through server-side knowledge sharing.}
\label{fig:pco_fl_scenario}
\end{figure*}

Motivated by this gap, we consider a practical and overlooked regime for pFL in which the task and label space remain fixed but clients join in onboarding steps over time, called Progressive Client Onboarding in Federated Learning (PCO‑FL). As illustrated in Figure~\ref{fig:pco_fl_scenario}, PCO‑FL organizes training into multi‑round communication within onboarding steps. At each step, a new group of clients trains locally on the fixed task and uploads updates. The server aggregates these signals to refine a shared model that can be initialized for any client. The refined model is then broadcast back, enabling both new and existing clients to benefit from collective knowledge without sharing data~\cite{mcmahan2017communication}. Conceptually, PCO‑FL sits between two paradigms: like pFL, it targets personalization under non‑IID data, but unlike typical pFL it does not assume a static client cohort; like Federated Continual Learning (FCL)~\cite{yoon2021federated, yang2024federated, wang2024federated}, it introduces temporal non‑stationarity, but unlike FCL the non‑stationarity stems from client participation changes rather than task shifts. This combination emphasizes PCO‑FL as a novel and underexplored scenario which requires examination.

We focus on two main investigations. First, we compare the progressive onboarding setup with isolated training for new clients under the same compute and initialization. We observe that new clients consistently start stronger, indicating that the server's shared model provides a useful prior at onboarding. Second, we examine the longer‑term effect when existing clients are not retrained in later steps. Two issues recur: existing clients' accuracy drops as the server adapts to new clients, and updates for new clients clash with previously specialized behaviors due to limited shared capacity and gradient misalignment.

Building on these findings, we introduce \textbf{p}ersonalized \textbf{Fed}erated Learning via \textbf{D}ata-free \textbf{S}ub-\textbf{H}ypernetwork (\textbf{pFedDSH}) for PCO-FL. It uses a server-side hypernetwork to generate personalized weights and step-specific masks that preserve allocated capacity while reusing compatible neurons and adding new ones only when updates conflict. This step-aware allocation prevents interference and protects existing clients while giving new clients a strong start. To send improvements back without retraining or sharing data, pFedDSH runs server-side data-free replay: the server synthesizes inputs and fine-tunes the hypernetwork so gains propagate to existing clients.
Empirically, pFedDSH outperforms competing baselines while matching their per-round communication time and memory footprint, and its batch-specific masks preserve neurons for future onboarding steps. In summary, our contributions are fourfold:

\begin{itemize}
\item We introduce PCO‑FL, a practical regime where tasks and label spaces remain fixed but client cohorts evolve over time. This setting bridges pFL and FCL yet exposes unique challenges that have been largely overlooked.

\item Through two investigations, we demonstrate that progressive FL system consistently gives new clients stronger starts than isolation, but also reveals two critical limitations of PCO-FL: (i) existing clients suffer accuracy drops when the shared model adapts to newcomers, and (ii) freezing the shared model protects existing clients but prevents any transfer of new knowledge.

\item We propose pFedDSH, a hypernetwork‑based framework, incorporating a batch-specific masking strategy and a data‑free replay mechanism, designed to preserve knowledge and send improvements from new clients back to existing clients.
\item Extensive experimental results demonstrate pFedDSH's improvements in PCO‑FL, outperforming prior pFL and FCL baselines across multiple datasets.
\end{itemize}

\section{Related Work}
\label{sec:related}                         
\subsection{Personalized Federated Learning}
pFL methods fall into two groups. Local personalization adapts the global model at each client (e.g., FedPer~\cite{arivazhagan2019federated}, LG-FedAvg~\cite{liang2020think}) or balances global/local via meta-learning or regularization (pFedMe~\cite{t2020personalized}, Per-FedAvg~\cite{fallah2020personalized}, Ditto~\cite{li2021ditto}). Global personalization injects client specificity during aggregation, using hypernetworks (pFedHN~\cite{shamsian2021personalized}), representation decoupling (FedRep~\cite{collins2021exploiting}), or selective fine-tuning (FedSelect~\cite{tamirisa2024fedselect}). However, FedSelect-style approaches still require continual client-side updates, whereas in PCO-FL incumbents do not retrain; our batch-specific, server-managed masks plus server-side replay preserve earlier capacity and route improvements back to them.

Recent work has also considered the case of \textit{late or novel clients}. ODPFL-HN~\cite{amosy2024late} proposes on-demand \textit{unlabeled} personalization, where a hypernetwork initiates a model for a new client using only its unlabeled data. PeFLL~\cite{scott2024pefll} likewise uses a server-side hypernetwork to generate models for unseen clients, aiming to provide strong initializations and reduce client-side optimization. Both are related to our setting but differ crucially: PCO-FL assumes supervised onboarding with labeled data for new clients and requires performance improvement for existing clients who do not retrain, which neither ODPFL-HN nor PeFLL supports.

\subsection{Federated Continual Learning}
Unlike pFL, FCL tackles sequential tasks and alleviates catastrophic forgetting under FL's communication, heterogeneity, and privacy limits~\cite{yoon2021federated, nori2025federated}. Methods such as FedWeIT (parameter partitioning)~\cite{yoon2021federated}, FedCLASS (feature-level transfer)~\cite{wu2024federated}, and HR (synthetic replay)~\cite{nori2025federated} work when tasks or labels change. PCO-FL differs in that the task remains fixed while the client population evolves. Previously served clients are frozen, making it essential to maintain their performance without local retraining, a requirement unmet by existing FCL methods but directly addressed by our masking and replay strategy.

\section{Problem Formulation}
\label{sec:problem}

We consider a trusted server and an unbounded pool of clients. Training proceeds in \textit{onboarding steps} $t=1,\dots,T$. At step $t$, a disjoint batch of new clients $\mathcal{B}_t$ joins the system, and the cumulative set of onboarded clients is
\begin{equation}
\mathcal{C}_t  =  \mathcal{C}_{t-1}\,\cup\,\mathcal{B}_t,
\qquad 
\mathcal{B}_t\cap\mathcal{C}_{t-1}=\varnothing,
\qquad 
\mathcal{C}_0=\varnothing ,
\label{eq:cumulative_clients}
\end{equation}
and we use $\mathcal{C}_{<t} = \bigcup_{\tau<t}\mathcal{B}_\tau$ for the set of existing clients before step $t$.
For any client $k\in\mathcal{C}_t$, let the private dataset be $\mathcal{D}^k=\{(x_i^k,y_i^k)\}_{i=1}^{n_k}$ with inputs $x_i^k\in\mathbb{R}^F$ and labels $y_i^k\in\mathcal{Y}$. The label space $\mathcal{Y}$ is fixed across all steps, while client participation evolves according to Equation~\ref{eq:cumulative_clients}. Let $C=|\mathcal{Y}|$ denote the number of classes.

\subsection{Training Objective and Constraints}
Let $f(\cdot;\theta)$ be a model with parameters $\theta$. Each client $k$ minimizes the empirical risk
\begin{equation}
\mathcal{L}_k(\theta) = \frac{1}{n_k}\sum_{(x,y)\in\mathcal{D}^k}\ell \left(f(x;\theta),y\right),
\label{eq:local_loss}
\end{equation}
where $\ell$ is a standard supervised loss (e.g., cross-entropy).
The \textbf{PCO-FL objective} is to minimize the aggregate loss incurred by new clients at each step:
\begin{equation}
\min_{\{\text{server/client updates over steps}\}}
\quad \sum_{t=1}^{T} \sum_{k\in\mathcal{B}_t} \mathbb{E} \left[\mathcal{L}_k \big(\theta_{k}^{(t)}\big)\right]
\label{eq:objective_constraints}
\end{equation}
such that no raw data leaves any client, no re-training for existing clients, and performance stability across batches, i.e., $\forall k\in\mathcal{C}_{<t},\ \mathrm{Acc}_k^{(t)} \ge \mathrm{Acc}_k^{(t-1)}-\varepsilon$, where $\mathrm{Acc}_k^{(t)}$ is the evaluation accuracy of client $k$ after completing step $t$, and $\varepsilon \ge 0$ is an allowable tolerance.

\subsection{Communication Protocol}
Each onboarding step $t$ comprises $R$ communication rounds; each participating client performs $E$ local epochs per round. We denote the \textit{server state} (a personalization prior) before round $r$ at step $t$ by $W_{t,r}$.
%its concrete form is instantiated by our method in Sec.~\ref{sec:method}.

\textbf{(1) Onboarding Initialization (round $r{=}1$ at step $t$).}
The server broadcasts the current state $W_{t,1}\equiv W_{t-1, R}$ to all clients in $\mathcal{B}_t$. Previously onboarded clients $\mathcal{C}_{<t}$ \textit{do not} conduct local training at step $t$.

\textbf{(2) Local Updates (clients in $\mathcal{B}_t$).}
Each client $k\in\mathcal{B}_t$ initializes its local model from the received state and optimizes~\eqref{eq:local_loss} for $E$ epochs, e.g.,
\begin{equation}
\theta_{k}^{(r+1)}  =  \theta_{k}^{(r)}  -  \eta\,\nabla_{\theta}\,\mathcal{L}_k \left(\theta_{k}^{(r)}\right),
\label{eq:local_step}
\end{equation}
with local learning rate $\eta$. Only model updates or gradients are uploaded; raw samples are never shared.

\textbf{(3) Global Aggregation (server).}
The server aggregates updates from $\mathcal{B}_t$ to obtain the next state $W_{t,r+1}$. A standard data-size-weighted update is
\begin{equation}
W_{t,r+1}  = 
\sum_{k\in\mathcal{B}_t}\frac{n_k}{\sum_{j\in\mathcal{B}_t}n_j} \theta_{k}^{(r+1)}.
\label{eq:aggregation}
\end{equation}
Phases (2)-(3) repeat $R$ times within step $t$, yielding $W_t\equiv W_{t,R+1}$. The server then proceeds to step $t{+}1$ with a new batch $\mathcal{B}_{t+1}$. The process terminates after the final round of the last step $T$.

\subsection{Evaluation Protocol}
\label{sec:eval_protocol}
To capture the goals of PCO-FL, we report:

\textbf{Proactive Adaptation (PA).}
For step $t$, the average accuracy gain of $\mathcal{B}_t$ over a local-only baseline with the same compute budget:
\begin{equation}
\mathrm{PA}_t  =  
\frac{1}{|\mathcal{B}_t|}\sum_{k\in\mathcal{B}_t}\Big(\mathrm{Acc}_k^{(t)} - \mathrm{Acc}_{k}^{\,\mathrm{local}}\Big).
\label{eq:pa}
\end{equation}

\textbf{Retroactive Improvement (RI).}
The average accuracy gain for existing clients after integrating $\mathcal{B}_t$ \textit{without} any new local training on those clients:
\begin{equation}
\mathrm{RI}_t  =  
\frac{1}{|\mathcal{C}_{<t}|}\sum_{k\in\mathcal{C}_{<t}}\Big(\mathrm{Acc}_k^{(t)} - \mathrm{Acc}_k^{(t-1)}\Big).
\label{eq:ri}
\end{equation}

\section{Empirical Investigation}
\label{sec:empirical}

Before introducing our method, we empirically investigate the behaviour of PCO-FL around two questions: 
\textbf{Q1} does the federated system improve onboarding performance for new clients compared to training in isolation,
and \textbf{Q2} what happens to already served clients when only new clients continue training at later steps.
We use the metrics $\mathrm{PA}_t$ and $\mathrm{RI}_t$ defined in Section~\ref{sec:eval_protocol}.

\textbf{Setup.}
We consider image classification on CIFAR-10 with a fixed label space and $T{=}5$ onboarding steps.
The initial step has $80$ clients and each later step introduces $5$ new clients, following the same non-IID split and communication protocol used in Section~\ref{sec:experiments}.
At step $t$, only the new clients $\mathcal{B}_t$ perform local updates, while existing clients $\mathcal{C}_{<t}$ stay frozen.
We instantiate representative methods from three families:
traditional pFL (FedPer, FedRep, FedSelect),
late-client-joining hypernetwork methods (pFedHN, PeFLL, ODPFL-HN),
and FCL-style methods adapted to this setting (FedWeIT, FedCLASS, HR).
For each method, we compare isolation versus federation under identical initialization and compute.
Full per-step numeric results are reported in the Appendix.

\textbf{Investigation I: updating the shared model with newcomers only.}
In the first experiment, the server keeps updating the shared model at step $t$ using only updates from $\mathcal{B}_t$.
Across all methods, we consistently observe that the newcomers benefit from federation.
For every $t$, $\mathrm{PA}_t$ is positive, typically in the range $2$ to $5$ points at early steps and reaching around $3$ to $6$ points by $t{=}5$.
This confirms that the shared state acts as a useful prior for new clients, answering Q1 positively.
However, this comes at a cost for existing clients.
For all methods, $\mathrm{RI}_t$ is negative once new batches start joining, and the cumulative loss for the first batch after five steps ranges roughly from $3$ to more than $15$ points, depending on the method.
In other words, progressive updating helps new clients but consistently drifts away from the interests of incumbents.

\textbf{Investigation II: freezing the shared model.}
In the second experiment, the server freezes the shared model once the first batch has finished training.
At step $t$, the current global model is broadcast to $\mathcal{B}_t$, which fine-tune locally without sending updates back for aggregation, while $\mathcal{C}_{<t}$ remain frozen as before.
By construction, this yields $\mathrm{RI}_t{=}0$ at every step: existing clients are neither harmed nor improved.
New clients still benefit from onboarding, but less aggressively than in Investigation I.
Across methods, $\mathrm{PA}_t$ remains positive, typically about $0.7$ to $1.5$ points, because the frozen global model already encodes structure learned from earlier clients, yet cannot adapt to the new data.

In summary, the two investigations lead to a clear picture.
Progressive aggregation with newcomers (Investigation I) gives strong onboarding gains for new clients (positive $\mathrm{PA}_t$) but inevitably hurts existing clients (negative $\mathrm{RI}_t$).
Freezing the server (Investigation II) protects existing clients but leaves them unable to benefit from new knowledge ($\mathrm{RI}_t{=}0$), and onboarding gains shrink.
This answers Q1 and Q2 jointly and motivates the design of pFedDSH:
we need a mechanism that preserves the performance of earlier clients while still letting improvements from new batches flow backward, without retraining incumbents or exposing their data.

\section{Methodology}
\label{sec:method}

\subsection{Overview of pFedDSH}
Guided by the findings in Section~\ref{sec:empirical}, \textbf{pFedDSH} is built to counter the two limitations observed under PCO-FL: \textit{(i): existing clients lose accuracy} when the shared model is updated using only newcomers, and \textit{(ii) newcomer updates conflict with behaviors relied on by existing clients}. To address these, Figure~\ref{fig:pfeddsh_training_flow} summarizes pFedDSH, with three components that work together:

\begin{itemize}[leftmargin=1.35em]
\item a \textbf{central hypernetwork} that generates full personalized weights at the server from client embeddings, providing strong onboarding initializations while keeping personalization under server control;
\item \textbf{batch-specific binary masks} that allocate and preserve subnetwork capacity across onboarding steps, freezing capacity already used by existing clients and routing conflicting updates into newly allocated channels;
\item \textbf{server-side data-free replay} that synthesizes inputs aligned with the most recent step and selectively fine-tunes the hypernetwork so improvements from new clients are sent to existing clients without retraining.
\end{itemize}

\subsection{Hypernetwork and Personalization}
\label{sec:hyper}
A central hypernetwork $H(\cdot;\phi)$~\cite{shamsian2021personalized} is the natural fit for PCO-FL. First, it yields instant personalization for new clients at onboarding: given an embedding $e_c$, the server produces a model $\theta_c = H(e_c;\phi)$ that already encodes cross-step knowledge. Second, keeping personalization at the server gives a single control point to freeze capacity used by existing clients and run server-side replay. Client-side fine-tuning approaches (e.g., FedPer~\cite{arivazhagan2019federated}, and FedRep~\cite{collins2021exploiting}) store personalization locally, which prevents the server from replay and limits backward transfer.

\textbf{Formulation and communication.}
Each client $c$ is represented by an embedding $e_c \in \mathbb{R}^{d}$. The hypernetwork maps $e_c$ to a parameter vector $\theta_c = H(e_c;\phi)\in\mathbb{R}^{P}$ for the base architecture. During the onboarding step $t$, only new clients $\mathcal{B}_t$ train, and their gradients backpropagate through $\theta_c$ into $\phi$ and are aggregated at the server. This design preserves privacy, yields strong starts for new clients, and prepares the server to harmonize updates across steps.

\subsection{Batch-Specific Masking}
\label{sec:mask}
Section~\ref{sec:empirical} shows that updating the shared model using only newcomers improves onboarding but degrades existing clients, revealing a capacity conflict. We resolve this by partitioning and preserving capacity across onboarding steps via binary masks that are learned for the current step and frozen thereafter.

For each onboarding step $t$ and layer $\ell$ (with $C_\ell$ channels), we maintain learnable mask logits $z_t^{(\ell)} \in \mathbb{R}^{C_\ell}$ and obtain a soft gate via a sigmoid with scale $\gamma>0$ (we set $\gamma=5000$ in practice):
\begin{equation}
g_t^{(\ell)}  =  \sigma \big(\gamma\, z_t^{(\ell)}\big) \in [0,1]^{C_\ell}.
\label{eq:sigmoid_gate}
\end{equation}

We record the capacity allocated up to step $t$ with a cumulative (element‑wise) union:
\begin{equation}
m_{\le t}^{(\ell)}  =  \max \big(m_{\le t-1}^{(\ell)},\, g_t^{(\ell)}\big),
\quad m_{\le 0}^{(\ell)}\equiv 0,
\label{eq:cum_union}
\end{equation}
where $\max$ is taken channel‑wise. Because $g_t^{(\ell)} \in [0,1]$ and becomes near‑binary as $\gamma$ grows, $m_{\le t}^{(\ell)}$ also remains in $[0,1]$ and behaves like a binary union in practice. Let $\theta_c^{(\ell)}$ be the hypernetwork‑generated weights for client $c$, the effective masked parameters at step $t$ are
\begin{equation}
\tilde{\theta}_{c,t}^{(\ell)}  =  m_{\le t}^{(\ell)} \odot \theta_c^{(\ell)},
\quad
f_{c,t}(x) =  f \big(x;\{\tilde{\theta}_{c,t}^{(\ell)}\}_\ell\big),
\label{eq:effective_params}
\end{equation}
with $\odot$ denoting element‑wise multiplication. This allows reuse of previously allocated channels while preventing their overwrite once frozen.

To avoid destructive updates on already allocated capacity, we modulate gradients so that a connection is updated only if it involves newly active channels at step $t$. For a weight $\theta_{ij}^{(\ell)}$ (from neuron $j$ in layer $\ell - 1$ to $i$ in layer $\ell$), we apply a soft gating factor:
\begin{align}
\frac{\partial \mathcal{L}}{\partial \theta_{ij}^{(\ell)}}
 = 
\frac{\partial \mathcal{L}}{\partial \theta_{ij}^{(\ell)}}
\cdot
\max \Big(
g_t^{(\ell)}(i)\cdot \big(1-m_{<t}^{(\ell)}(i)\big), \nonumber \\
g_t^{(\ell-1)}(j)\cdot \big(1-m_{<t}^{(\ell-1)}(j)\big)
\Big),
\label{eq:grad_gate}
\end{align}
where $m_{<t}$ is the cumulative mask from previous steps. This makes gradients zero when both endpoints were already allocated before step $t$, thereby preventing overwrite.

\begin{figure}[t]
\centering
\includegraphics[width=\linewidth]{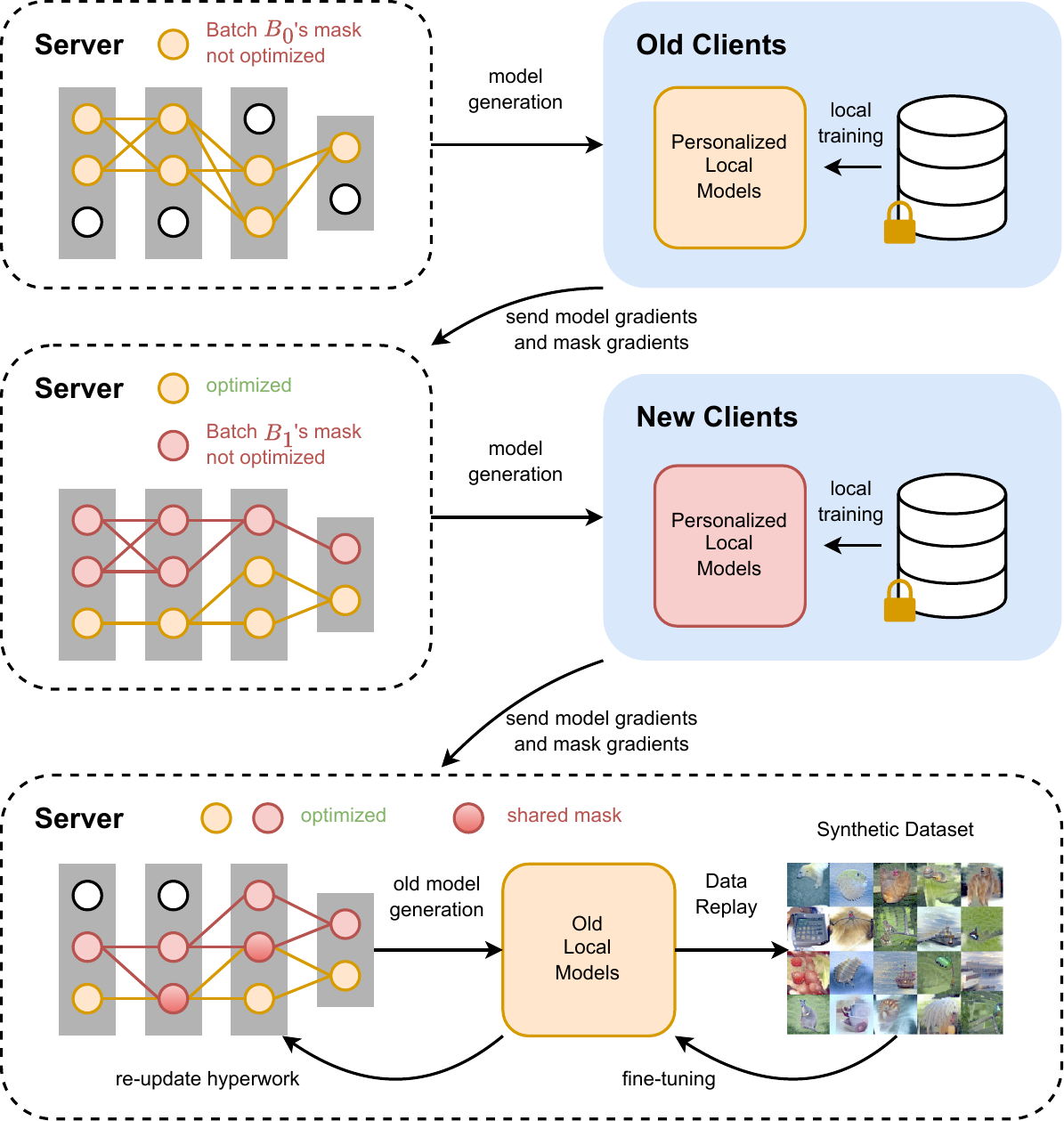}
\caption{Training flow of pFedDSH. Clients receive personalized models from the hypernetwork with batch-specific masks. Gradients from clients are aggregated at server. Synthetic images generated via DeepInversion facilitate knowledge transfer, allowing for fine-tuning of the global model without requiring additional data.}
\label{fig:pfeddsh_training_flow}
\end{figure}

\subsection{Server-side Data-free Replay}
\label{sec:replay}
Masks prevent harm but do not improve existing clients on their own. Because existing clients do not retrain and their data are unavailable, the server needs a privacy-compatible way to send useful step-$t$ information backward, and DeepInversion~\cite{yin2020dreaming} fulfills this desire.

Let $\tilde{x}$ denote a synthetic image. The first component is a \textit{feature-distribution regularizer} that matches the mini-batch mean and variance of activations at each BN layer to the corresponding running statistics recorded during client training. Let $\mu_l(\tilde{x})$ and $\sigma_l^2(\tilde{x})$ for the mean and variance of activations induced by $\tilde{x}$ at BN layer $l$, and $\{\hat{\mu}_l,\hat{\sigma}_l^2\}$ for the stored BN statistics, we define
\begin{equation}
% \begin{aligned}
R_{\mathrm{feature}}(\tilde{x})=
\sum_{l}(\Big\|\mu_l(\tilde{x})-\hat{\mu}_l\Big\|_2^2 + 
\Big\|\sigma_l^2(\tilde{x})-\hat{\sigma}_l^2\Big\|_2^2).
% \end{aligned}
\end{equation}
To encourage piecewise-smooth natural image structure, we add a total-variation prior. Following DeepInversion's formulation, we compute the sum of norms between the image and its one-pixel shifted variants in horizontal, vertical, and diagonal directions ($d\in\{(1,0),(0,1),(1,1),(1,-1)\}$):
\begin{equation}
R_{\mathrm{TV}}(\tilde{x})
 = 
\sum_{d}
\big\|\tilde{x}-\mathrm{shift}_{d}(\tilde{x})\big\|_{2},
\end{equation}
This finite-difference prior stabilizes optimization and reduces checkerboard patterns during synthesis.

Finally, a mild pixel-norm penalty constrains global amplitude and prevents degenerate solutions:
\begin{equation}
R_{L_2}(\tilde{x})  =  \|\tilde{x}\|_2^2.
\end{equation}
In general, we solve the following optimization problem:
\begin{equation}
\begin{aligned}
\tilde{x}^{*}
 &= 
\arg\min_{\tilde{x}}
\beta_{\mathrm{feat}} R_{\mathrm{feature}}(\tilde{x})
 \\&+ 
\beta_{\mathrm{TV}} R_{\mathrm{TV}}(\tilde{x})
 + 
\beta_{L_2} R_{L_2}(\tilde{x}),
\end{aligned}
\end{equation}
where $\beta_{\mathrm{feat}},\beta_{\mathrm{TV}},\beta_{L_2} > 0$ balance distribution alignment and natural-image priors.

After synthesizing images to form the data-free Synthetic Data Pool $\mathcal{S}_t$ for the current batch $t$, we fine-tune the global model $\theta^{(t)}$ at the server side. Specifically, for each previous batch $t' < t$, we retrieve the corresponding subnetwork, defined by batch-specific neuron masks $m_{t'}$, and perform fine-tuning using the synthetic dataset $\mathcal{S}_t$:
\begin{equation}
\theta^{(t')} = \arg\min_{\theta^{(t')}} \sum_{(\tilde{x}, y) \in \mathcal{S}_t} \mathcal{L}\left(f(\tilde{x}; \theta^{(t')} \odot m_{t'}), y\right)
\label{eq:finetune}
\end{equation}
where $f(\tilde{x}; \theta^{(t')} \odot m_{t'})$ denotes the subnetwork obtained by element-wise multiplying the global parameters $\theta^{(t')}$ with the batch-specific mask $m_{t'}$, and $\mathcal{L}(\cdot, \cdot)$ is the standard task-specific loss (e.g., cross-entropy loss). By performing this fine-tuning, we selectively propagate the updated global knowledge derived from new clients backward to existing clients, thereby facilitating knowledge transfer without the need of local retraining. Finally, we summarize our method in pseudo-code in the Appendix.

\section{Experiments}
\label{sec:experiments}

Our experiments are organized around three objectives.
Our experiments are built around three goals.
(i) pFedDSH should give new clients stronger starts (high PA) while also improving, not harming, the clients that arrived earlier (positive RI) under realistic onboarding schedules.
(ii) It make good use of the network, reusing capacity across batches instead of carving out a separate subnetwork for each one.
(iii) It achieve these benefits without paying a large price in communication cost or server runtime.

\subsection{Experimental Setup}
\label{sec:exp_setup}

\textbf{Datasets and PCO-FL schedules.}
We evaluate on CIFAR-10, CIFAR-100~\cite{krizhevsky2009learning}, and Tiny-ImageNet~\cite{le2015tiny}.
In all cases, the task and label space are fixed while clients arrive over time.
We construct $100$ clients with a Dirichlet label partition ($\alpha=0.1$) and use the same non-IID split for training and test.
The first onboarding step serves $80$ clients.
We then consider three schedules:
(i) \textbf{2 batches} (one additional batch of $20$ clients),
(ii) \textbf{5 batches} (four additional batches of $5$ clients),
and (iii) \textbf{11 batches} (ten additional batches of $2$ clients).
PA and RI are computed as in Section~\ref{sec:eval_protocol}.

\textbf{Baselines.}
We compare against three families.
(i) \emph{Traditional pFL:} FedPer, FedRep, FedSelect, which use shared backbones with client-specific heads or decoupled representations.
(ii) \emph{Late-client-joining hypernetworks:} pFedHN, PeFLL, ODPFL-HN, which explicitly generate models for unseen clients but do not target backward improvement.
(iii) \emph{FCL-style methods:} FedWeIT, FedCLASS, HR, adapted to PCO-FL as strong continual-learning baselines.
FedAvg~\cite{mcmahan2017communication} is included as a standard FL reference.

\textbf{Training protocol.}
Unless otherwise specified, the first batch is trained for $200$ communication rounds and each subsequent batch for $100$ rounds.
At each round, we sample $5\%$ of currently active clients.
Clients train locally for one epoch per round using SGD with momentum $0.9$, batch size $32$, and an initial learning rate $0.01$ with cosine decay.
Local-only baselines use the same initialization and local compute as the onboarding runs, with models reset at joining time.
Architectural details, DeepInversion configuration, and hardware are given in the Appendix.

\subsection{Positive PA and RI}
\label{sec:exp_obj1}

We first examine whether pFedDSH simultaneously improves new clients (PA) and existing clients (RI) under different onboarding schedules.

\begin{table}[t]
\centering
\footnotesize
\caption{RI (existing clients) and PA (new clients) for the last onboarding batch under PCO-FL with 2, 5, and 11 batches.
RI is the change in accuracy of existing clients before and after onboarding the last batch; PA is the gain of the last batch over isolated training.
Positive RI indicates retroactive improvement.
Means over five seeds.
Full per-batch tables are in the Appendix.}
\label{tab:acc}
\resizebox{\linewidth}{!}{
\begin{tabular}{llrrrrrr}
\toprule
\multirow{2}{*}{Type} & \multirow{2}{*}{Method} 
 & \multicolumn{2}{c}{2 batches} 
 & \multicolumn{2}{c}{5 batches} 
 & \multicolumn{2}{c}{11 batches} \\ 
\cmidrule(lr){3-4} \cmidrule(lr){5-6} \cmidrule(lr){7-8}
 & & RI & PA & RI & PA & RI & PA \\ \midrule
\multicolumn{8}{c}{\textbf{CIFAR-10}} \\ \midrule
FL     & FedAvg      & -18.56 & -13.21 & -22.50 & -14.00 & -26.00 & -15.00 \\ \cmidrule(lr){1-8}
\multirow{3}{*}{pFL} 
       & FedPer      & -7.09  & 12.40  & -9.09  & 11.60  & -10.59 & 11.00  \\
       & FedRep      & -11.68 & 13.00  & -13.68 & 12.20  & -15.18 & 11.60  \\
       & FedSelect   & -6.13  & 14.29  & -8.13  & 13.49  & -9.63  & 12.89  \\ \cmidrule(lr){1-8}
\multirow{3}{*}{LCJ} 
       & pFedHN      & -9.73  & 13.16  & -12.73 & 12.16  & -14.73 & 11.16  \\
       & PeFLL       & -12.19 & 15.41  & -15.19 & 14.41  & -17.19 & 13.41  \\
       & ODPFL-HN    & -6.27  & 14.13  & -9.27  & 13.13  & -11.27 & 12.13  \\ \cmidrule(lr){1-8}
\multirow{3}{*}{FCL} 
       & FedWeIT     & -6.26  & 10.50  & -7.76  & 9.50   & -8.76  & 8.70   \\
       & FedCLASS    & -5.97  & 11.15  & -7.47  & 10.15  & -8.47  & 9.35   \\
       & HR          & -6.08  & 11.40  & -7.58  & 10.40  & -8.58  & 9.60   \\ \cmidrule(lr){1-8}
\textbf{Ours} 
       & \textbf{pFedDSH} 
                     & \textbf{3.80}  &  \textbf{14.53} 
                     & \textbf{5.10}  &  \textbf{14.00} 
                     & \textbf{6.20}  &  \textbf{13.50} \\ \midrule
\multicolumn{8}{c}{\textbf{CIFAR-100}} \\ \midrule
FL     & FedAvg      & -9.99  & -16.06 & -14.00 & -17.00 & -18.00 & -18.00 \\ \cmidrule(lr){1-8}
\multirow{3}{*}{pFL} 
       & FedPer      & -6.17  & 3.50   & -9.17  & 2.70   & -12.17 & 1.90   \\
       & FedRep      & -10.89 & -3.26  & -13.89 & -4.06  & -16.89 & -4.86  \\
       & FedSelect   & -7.56  & 7.59   & -10.56 & 6.79   & -13.56 & 5.99   \\ \cmidrule(lr){1-8}
\multirow{3}{*}{LCJ} 
       & pFedHN      & -22.48 & 7.85   & -26.48 & 6.65   & -30.48 & 5.45   \\
       & PeFLL       & -19.23 & 4.59   & -23.23 & 3.39   & -27.23 & 2.19   \\
       & ODPFL-HN    & -7.90  & 7.10   & -11.90 & 5.90   & -15.90 & 4.70   \\ \cmidrule(lr){1-8}
\multirow{3}{*}{FCL} 
       & FedWeIT     & -4.43  & 1.56   & -6.43  & 1.16   & -8.43  & 0.76   \\
       & FedCLASS    & -4.07  & 1.81   & -6.07  & 1.41   & -8.07  & 1.01   \\
       & HR          & -3.94  & 1.96   & -5.94  & 1.56   & -7.94  & 1.16   \\ \cmidrule(lr){1-8}
\textbf{Ours} 
       & \textbf{pFedDSH} 
                     & \textbf{3.50}  & \textbf{8.37} 
                     & \textbf{4.22}  & \textbf{7.88} 
                     & \textbf{5.02}  & \textbf{7.39} \\ \midrule
\multicolumn{8}{c}{\textbf{Tiny-ImageNet}} \\ \midrule
FL     & FedAvg      & -0.09  & -9.31  & -0.99  & -9.80  & -1.99  & -10.30 \\ \cmidrule(lr){1-8}
\multirow{3}{*}{pFL} 
       & FedPer      & -3.31  & -0.83  & -4.31  & -1.13  & -5.31  & -1.43  \\
       & FedRep      & -2.13  & -1.96  & -3.13  & -2.26  & -4.13  & -2.56  \\
       & FedSelect   & -3.30  & 2.00   & -4.30  & 1.70   & -5.30  & 1.40   \\ \cmidrule(lr){1-8}
\multirow{3}{*}{LCJ} 
       & pFedHN      & -9.81  & -1.08  & -11.81 & -1.58  & -13.81 & -2.08  \\
       & PeFLL       & -3.20  & -1.00  & -4.20  & -1.30  & -5.20  & -1.60  \\
       & ODPFL-HN    & -3.20  & -0.31  & -4.20  & -0.61  & -5.20  & -0.91  \\ \cmidrule(lr){1-8}
\multirow{3}{*}{FCL} 
       & FedWeIT     & -1.50  & -0.57  & -2.20  & -0.77  & -2.90  & -0.97  \\
       & FedCLASS    & -1.28  & -0.39  & -1.98  & -0.59  & -2.68  & -0.79  \\
       & HR          & -1.55  & -0.48  & -2.25  & -0.68  & -2.95  & -0.88  \\ \cmidrule(lr){1-8}
\textbf{Ours} 
       & \textbf{pFedDSH} 
                     & \textbf{1.80}  & \textbf{3.27} 
                     & \textbf{3.20}  & \textbf{3.00} 
                     & \textbf{5.60}  & \textbf{5.13} \\ \bottomrule
\end{tabular}
}
\end{table}

\textbf{Summary across datasets and schedules.}
Table~\ref{tab:acc} reports RI for existing clients and PA for newly onboarded clients in all three schedules.
Across datasets, all baselines suffer negative RI once later batches are introduced, and existing clients steadily lose accuracy as the system continues to onboard.
In contrast, pFedDSH is the only method that achieves consistently positive RI while also preserving strong PA.
Moreover, its RI tends to increase as more batches arrive: when additional clients join, replay has more diverse updates to consolidate, so existing clients benefit more from later onboarding steps instead of being harmed by them.
At the same time, onboarding clients still obtain a clear advantage over training in isolation, showing that these backward gains do not come at the expense of their own performance.

\textbf{Per-batch trajectories.}
Figure~\ref{fig:acc} shows the average test accuracy per batch in the 5-batch setting.
Traditional pFL and FCL baselines exhibit a clear downward trend for early batches once later clients start training.
Figure~\ref{fig:acc_first} makes this explicit by tracking the \emph{first} batch across 11 onboarding steps: for most baselines, the first batch deteriorates steadily as more batches are added, whereas under pFedDSH it remains almost flat and can even gain a small improvement after replay.
We sometimes see a small decrease for pFedDSH in a few batches, which is due to the masks being almost, reflecting that masks are near-binary rather than perfectly binary, making tiny gradient leakage can occur.
However, this effect is small compared to the degradation observed for other methods.

\begin{figure}[h]
    \centering
    \includegraphics[width=\linewidth]{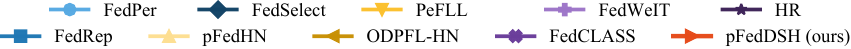}\\
    \vspace{0.2cm}
    \includegraphics[width=0.8\linewidth]{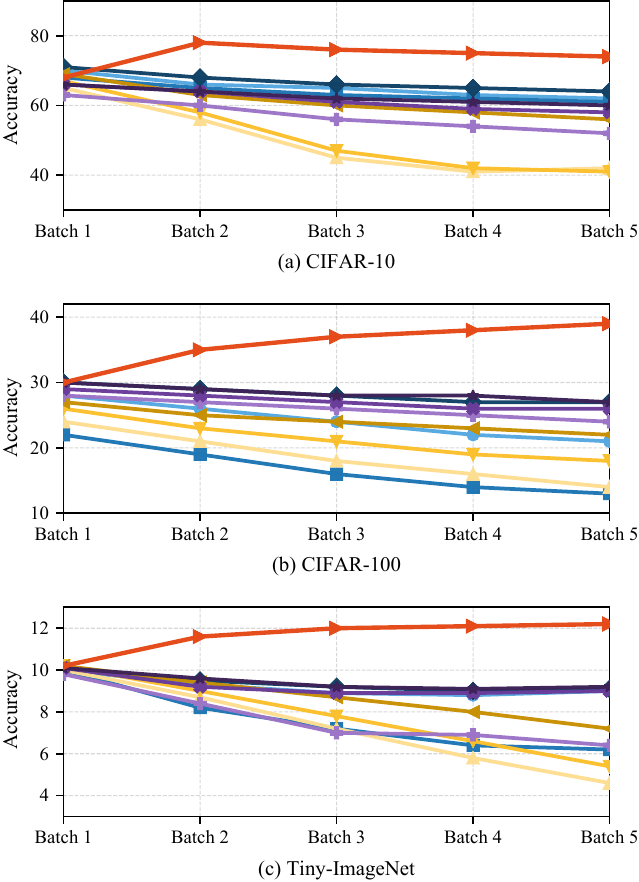}
    \caption{Average test accuracy per batch in the 5-batch setting.
    pFedDSH keeps earlier batches stable (and can slightly improve them) as new batches arrive, while baselines typically degrade earlier batches as onboarding continues.}
    \label{fig:acc}
\end{figure}

\begin{figure}[t]
  \centering
  \begin{subfigure}[t]{0.8\linewidth}
    \centering
    \includegraphics[width=\linewidth]{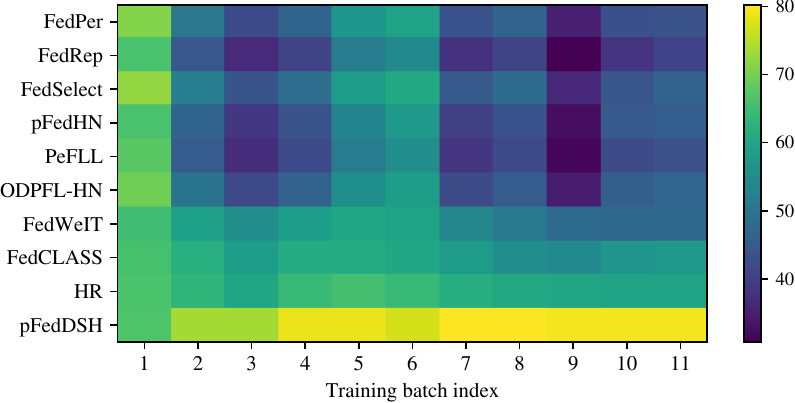}
    \caption{CIFAR-10}
    \label{fig:acc_first:c10}
  \end{subfigure}
  \hfill
  \begin{subfigure}[t]{0.8\linewidth}
    \centering
    \includegraphics[width=\linewidth]{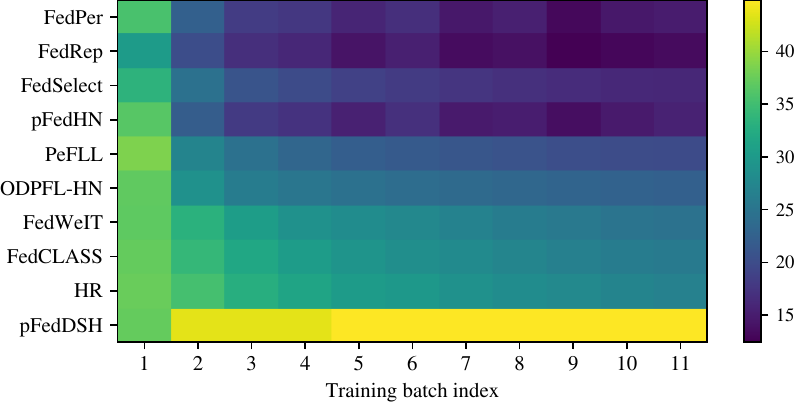}
    \caption{CIFAR-100}
    \label{fig:acc_first:c100}
  \end{subfigure}
  \hfill
  \begin{subfigure}[t]{0.8\linewidth}
    \centering
    \includegraphics[width=\linewidth]{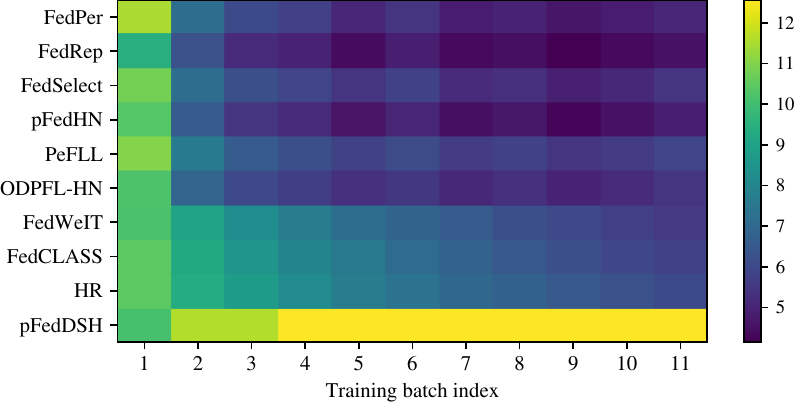}
    \caption{Tiny-ImageNet}
    \label{fig:acc_first:tiny}
  \end{subfigure}
  \vspace{0.4em}
  \caption{Accuracy of the first batch measured after each subsequent onboarding batch (up to 11).
  pFedDSH maintains near-constant performance for the first batch and can even improve it via server-side replay, whereas competing methods show gradual degradation over time.}
  \label{fig:acc_first}
\end{figure}

\textbf{Ablation study.}
To understand which components matter, we ablate batch-specific masking and data-free replay on CIFAR-10 with five batches (Table~\ref{tab:ablation_components}).
Removing \textbf{neuron masking} makes RI clearly negative and substantially lowers accuracy, behaving similarly to standard hypernetworks where later updates overwrite earlier ones.
Removing \textbf{data-free replay} keeps existing clients relatively stable but largely removes backward gains, confirming that replay is what sends improvements back to incumbents.
Turning off both masking and replay gives the weakest configuration.
Taken together with Table~\ref{tab:acc}, this indicates that masks provide stability, while server-side replay turns that stability into consistently positive RI instead of leaving it near zero.

\begin{table}[ht]
\centering
\caption{Influence of individual components on average accuracy, PA, and RI after training on CIFAR-10 (5 batches).}
\label{tab:ablation_components}
\resizebox{0.473\textwidth}{!}{
\begin{tabular}{lccc}
\toprule
Variant & Accuracy (\%) & PA (\%) & RI (\%) \\
\midrule
\textbf{Full pFedDSH}          & \textbf{68.40} & 1.47 & \textbf{2.12}\\
  w/o neuron masking           & 51.22 & 1.55 & -14.31\\
  w/o data-free replay         & 64.15 & 1.42 & -0.92\\
  w/o masking \& replay        & 50.97 & 1.53 & -15.23\\
\bottomrule
\end{tabular}
}
\end{table}

\begin{figure}[ht]
    \centering
    \includegraphics[scale=0.57]{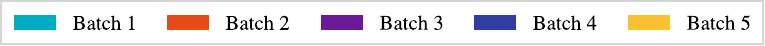}\\
    \vspace{0.3cm}
     \begin{subfigure}[b]{0.3\linewidth}
         \centering
         \includegraphics[width=\linewidth]{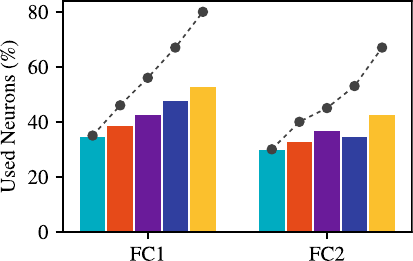}\\
         \caption{CIFAR-10}
         \label{fig:neuron:cifar10}
     \end{subfigure}
     \hfill
     \begin{subfigure}[b]{0.3\linewidth}
         \centering
         \includegraphics[width=\linewidth]{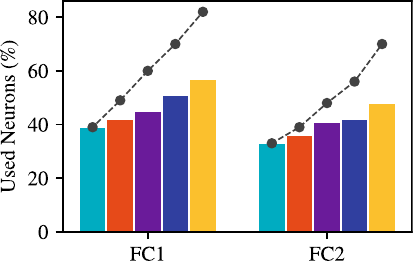}\\
         \caption{CIFAR-100}
         \label{fig:neuron:cifar100}
     \end{subfigure}
     \hfill
     \begin{subfigure}[b]{0.3\linewidth}
         \centering
         \includegraphics[width=\linewidth]{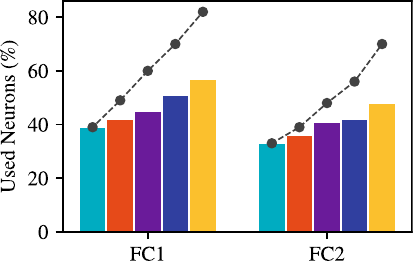}\\
         \caption{Tiny-ImageNet}
         \label{fig:neuron:tinyimg}
     \end{subfigure}
        \caption{Layer-wise neuron usage through the training process. Dashed lines show total active neurons; colored curves show per-batch usage. pFedDSH gradually increases total capacity while reusing a large fraction of earlier neurons.}
        \label{fig:neuron}
\end{figure}

\subsection{Efficient Neuron Capacity Usage}
\label{sec:exp_obj2}

We next examine whether pFedDSH simply allocates a separate subnetwork per batch or actually reuses neurons across onboarding steps.
Figure~\ref{fig:neuron} tracks layer-wise neuron usage over 5 batches for all three datasets, focusing on the fully connected layers of LeNet-5 for clarity.
First, the total number of active neurons (dashed line) does increase as more batches arrive, but the growth is gradual and far below what a “one independent subnetwork per batch” strategy would require, indicating that pFedDSH does not rely on disjoint partitions.
Second, the per-batch curves show increasing reuse: as new batches join, their masks activate many neurons that were already used by previous batches and only a smaller subset of newly allocated channels.
On CIFAR-10 and CIFAR-100, up to roughly $40\%$ of channels are reused across batches while maintaining accuracy; Tiny-ImageNet follows the same trend with slightly higher total usage due to the larger label space.
This behaviour matches the structure of PCO-FL: although client cohorts change, all batches tackle the same classification task, therefore the masking mechanism prefers to reuse compatible neurons and only allocates new capacity when incoming gradients introduce novel directions.

\subsection{Communication Efficiency}
\label{sec:exp_obj3}

Finally, we evaluate whether the benefits of pFedDSH come with acceptable efficiency in terms of server latency and communication cost.

\textbf{Server latency and trade-off.}
Table~\ref{tab:latency_vs_benefit} reports server-side latency per round on CIFAR-10 together with RI.
Shared-backbone pFL and hypernetwork baselines have comparable per-round costs.
HR is the only method that increases latency, since it runs replay at every round while still yielding negative RI.
pFedDSH has a similar base cost to other hypernetworks, but runs replay only once per batch, so the amortized overhead remains modest.
Under this latency budget, pFedDSH is the only method that attains positive RI, giving a clear trade-off: slightly more server work, but a qualitatively different outcome for existing clients.

\begin{table}[t]
\centering
\footnotesize
\caption{Server latency and RI on CIFAR-10 (5 batches).
``Ordinary round'' is the per-round cost without replay.
For HR, replay runs every round; for pFedDSH, replay runs once per batch and is amortized into the averaged cost.}
\label{tab:latency_vs_benefit}
\resizebox{\linewidth}{!}{
\begin{tabular}{lcccc}
\toprule
Method &
\begin{tabular}[c]{@{}c@{}}Ordinary\\ round (s)\end{tabular} &
\begin{tabular}[c]{@{}c@{}}Replay\\ latency (s)\end{tabular} &
\begin{tabular}[c]{@{}c@{}}Averaged\\ per round (s)\end{tabular} &
RI (\%) \\ \midrule
FedPer     & 0.32 & - & 0.32 & -14.23 \\
FedRep     & 0.32 & - & 0.32 & -15.47 \\
FedSelect  & 0.33 & - & 0.33 & -11.06 \\ \midrule
pFedHN     & 0.38 & - & 0.38 & -9.73  \\
PeFLL      & 0.39 & - & 0.39 & -10.66 \\
ODPFL-HN   & 0.40 & - & 0.40 & -11.77 \\ \midrule
FedWeIT    & 0.36 & - & 0.36 & -6.58 \\
FedCLASS   & 0.37 & - & 0.37 & -5.81 \\
HR         & 0.31 & 2.21 (every round) & 2.52 & -3.90 \\ \midrule
\textbf{pFedDSH} & 0.38 & 8.00 (once per batch) & \textbf{0.42} & \textbf{+2.12} \\ \bottomrule
\end{tabular}
}
\end{table}

\textbf{Communication per round.}
Table~\ref{tab:mem-per-round} reports the per-round communication payload per active client.
pFedDSH is bandwidth-equivalent to other hypernetwork-based methods and very close to traditional pFL, differing by only a few kilobytes per client per round from pFedHN and PeFLL.
Replay is implemented entirely on the server and does not introduce any additional client-side communication.
Overall, pFedDSH delivers positive RI and strong PA under PCO-FL with essentially unchanged communication cost and a modest, well-controlled latency overhead, which keeps it within a practical regime for real deployments.

\begin{table}[t]
\centering
\footnotesize
\caption{Per-round communication per active client.
S$\to$C and C$\to$S are server-to-client and client-to-server payloads.}
\label{tab:mem-per-round}
\resizebox{0.47\textwidth}{!}{
\begin{tabular}{lccc}
\toprule
\textbf{Method} & \textbf{S$\to$C (KiB)} & \textbf{C$\to$S (KiB)} & \textbf{Total (KiB)} \\ \midrule
FedPer          & 238.89                 & 238.89                 & 477.78               \\
FedRep          & 238.89                 & 238.89                 & 477.78               \\
FedSelect       & 238.89                 & 238.89                 & 477.78               \\ \midrule
pFedHN          & 242.21                 & 242.71                 & 484.92               \\
PeFLL           & 242.21                 & 242.71                 & 484.92               \\
ODPFL-HN        & 242.21                 & 242.71                 & 484.92               \\ \midrule
FedWeIT         & 240.00                 & 240.00                 & 480.00               \\
FedCLASS        & 241.00                 & 241.00                 & 482.00               \\
HR              & 242.21                 & 314.70                 & 556.91               \\ \midrule
pFedDSH         & 242.21                 & 243.59                 & 485.80               \\ \bottomrule
\end{tabular}
}
\end{table}

% =======================
% CONCLUSION
% =======================
\section{Conclusion}
\label{sec:conclusion}

In this paper, we studied the PCO-FL scenario, where tasks remain fixed but client cohorts evolve over time, and introduced PA and RI to quantify onboarding benefits and retroactive effects.
Our empirical analysis showed that progressive federated training gives new clients stronger starts than isolation, but, without explicit consolidation, existing clients either degrade or remain unchanged.
To address this, we proposed \textbf{pFedDSH}, a hypernetwork-based framework with batch-specific masking to preserve earlier knowledge and a server-side data-free replay mechanism to route improvements backward without exposing client data.
Experiments on multiple benchmarks demonstrate that pFedDSH improves both new and existing clients over state-of-the-art pFL and FCL methods, while reusing network capacity across onboarding steps and maintaining competitive communication and latency profiles.

\textbf{Privacy considerations.}
Unlike HR, which transmits latent codes through pretrained local decoders, pFedDSH synthesizes images purely on the server from aggregated BatchNorm statistics via DeepInversion~\cite{yin2020dreaming} and applies replay only on server-side embeddings.
Prior work indicates that reconstructions from such statistics are weaker than gradient-based attacks, making this form of replay comparatively more privacy-friendly.
Combined with the fact that raw client data and sample-level gradients never leave devices, pFedDSH offers a practical route to activate knowledge transfer in PCO-FL without introducing data leakage.

\section*{Acknowledgement}

This work was supported by the project ``Privacy-Preserving, Robust, and Explainable Federated Learning Framework for Healthcare System'' (Grant No. 8665) from the VinUni-Illinois Smart Health Center (VISHC), VinUniversity.

{
    \small
    \bibliographystyle{ieeenat_fullname}
    \bibliography{main}
}

\clearpage
\appendix

\begin{algorithm}[t]
\caption{\textbf{pFedDSH}: Personalized Federated Data-free Sub-Hypernetwork for PCO-FL}
\label{algo:pfeddsh}
\DontPrintSemicolon
\SetKwInOut{Input}{Input}\SetKwInOut{Output}{Output}

\Input{global hypernetwork $H(\cdot;\phi)$, total batches $T$, global epochs $R$, local epochs $E$, learning rates $\eta,\alpha$, masking regularizer $\lambda$.}
\Output{Personalized global hypernetwork $H(\cdot;\phi^{(T)})$, masks $\{m_t\}_{t=1}^{T}$.}

\For{$t \leftarrow 1$ \KwTo $T$}{

\For{$r \leftarrow 1$ \KwTo $R$}{

% Step 1: Client-server handshake
{\textcolor{red}{\tcp*[l]{\footnotesize \textbf{Step 1: Embedding Transmission}}}}
\ForEach{client $c \in \mathcal{B}_t$}{
    Client initializes local embedding $e_c$ and sends $(e_c, t)$ to server.
}

% Step 2: Server generates personalized masked models
{\textcolor{red}{\tcp*[l]{\footnotesize \textbf{Step 2: Model Generation}}}}
Server initializes or retrieves mask $m_t$.

Server generates personalized masked model for each client $\theta_c^{(t,0)} \leftarrow H(e_c;\phi)\odot m_t$ and sends it to corresponding clients $c\in\mathcal{B}_t$.

%------------------------------------------------
{\textcolor{red}{\tcp*[l]{\footnotesize \textbf{Step 3: Local Client Training}}}}
\ForEach{client $c \in \mathcal{B}_t$}{
    Train locally for $E$ epochs:
    $
    \theta_c^{(t)} \leftarrow \theta_c^{(t)} - \eta\nabla_{\theta_c^{(t)}} \mathcal{L}(\theta_c^{(t)}, \mathcal{D}_c)
    $.

    Send updated gradients $\nabla_{\phi}^{(c)}$ and mask gradients $\nabla_{m_t}^{(c)}$ to the server.
}

%------------------------------------------------
{\textcolor{red}{\tcp*[l]{\footnotesize \textbf{Step 2: Server Aggregation}}}}
Update global hypernetwork and mask:
$
\phi^{(t)} \leftarrow \phi^{(t-1)} - \alpha\sum_{c\in \mathcal{B}_t}\nabla_{\phi}^{(c)}$,
$
m_t \leftarrow m_t - \alpha\sum_{c\in \mathcal{B}_t}\nabla_{m_t}^{(c)}$.

}

%------------------------------------------------
{\textcolor{red}{\tcp*[l]{\footnotesize \textbf{Step 3: Replay via DeepInversion}}}}
\If{$t>1$}{
Generate synthetic dataset $\mathcal{S}_t$.

Fine-tune previous batches' subnetworks with synthetic data.
}

}
\end{algorithm}

\section{Implementation Details}
\label{sec:implementation_details}

\textbf{Model Architectures.}
Our experiments utilize \textbf{LeNet-5}~\cite{lecun2002gradient} as the backbone for all client-side networks. The final fully connected layer is adjusted to the number of classes in each dataset: $10$ for CIFAR-10, $100$ for CIFAR-100, and $200$ for Tiny-ImageNet. \textbf{The global hypernetwork} is a two-layer Multi-Layer Perceptron (MLP) that maps client embeddings of dimension $d=32$ through a hidden layer of $512$ neurons to generate the complete personalized parameter vector of LeNet-5. For generating client embeddings, we employ a lightweight convolutional neural network (CNN) consisting of one convolutional layer ($32$ filters, $3\times3$ kernels), followed by BN, a ReLU activation, global average pooling, and a final linear transformation projecting to the $32$-dimensional embedding space.

\textbf{Training Procedure.}
Federated training is conducted over $200$ communication rounds for initial client batches and an additional $100$ rounds for each new client batch. Each communication round randomly selects $5\%$ of currently available clients. Selected clients perform local training for one epoch per round using Stochastic Gradient Descent (SGD) with momentum set to $0.90$, a batch size of $32$, and an initial learning rate of $0.01$ following a cosine decay schedule throughout training. Clients transmit their embedding vectors, gradients with respect to hypernetwork parameters $\phi$, and batch-specific mask gradients to the server. No raw images, labels, or sample-specific gradients are exchanged. We summarize our method in the Algorithm \ref{algo:pfeddsh}.

\textbf{DeepInversion-based Data-free Replay.}
Synthetic images for server-side replay are generated via DeepInversion using BN statistics from the client models. The image synthesis optimization runs for $250$ iterations per image, with an initial learning rate of $0.1$. The feature alignment losses are weighted by the following coefficients: total variation regularization $\beta_{\text{TV}}=10^{-5}$, $L_2$ regularization $\beta_{L_2}=10^{-4}$, and feature distribution alignment $\beta_{\text{feature}}=10^{-2}$. To balance computational cost and diversity, the synthetic dataset per class per batch consists of $20$ images for CIFAR-10~\cite{krizhevsky2009learning}, $5$ images for CIFAR-100~\cite{krizhevsky2009learning}, and $3$ images for Tiny-ImageNet~\cite{le2015tiny}. Generated data are utilized solely for fine-tuning server-side global models.

\textbf{Hardware and Computational Resources.}
All reported computational performance measurements, including server-side latency, are obtained by averaging over $200$ federated training rounds executed on a single NVIDIA RTX A5000 GPU with $24$ GB of VRAM. Experiments are conducted on a computational node equipped with a 32-core AMD CPU running at $3.0$ GHz and $128$ GB of RAM. For fairness in comparison, no multi-GPU or distributed training optimizations are employed.

% \section{Ablation Study}
% \label{sec:ablation}

% To isolate the contribution of each architectural choice, we systematically disable (i) \emph{batch-specific neuron masking} and (ii) \emph{data-free replay}. We then conducted experiments on the five‑batch CIFAR‑10 setting.
% \begin{table}[ht]
% \centering
% \caption{Influence of individual components on average accuracy, PA, and RI after the training complete.}
% \label{tab:ablation_components}
% \resizebox{0.473\textwidth}{!}{
% \begin{tabular}{lccc}
% \toprule
% Variant & Accuracy (\%) & PA (\%) & RI (\%) \\
% \midrule
% \textbf{Full pFedDSH}          & \textbf{68.40} & 1.47 & \textbf{2.12}\\
%   w/o neuron masking         & 51.22 & 1.55 & -14.31\\
%   w/o data‑free replay       & 64.15 & 1.42 & -0.92\\
%   w/o masking \& replay      & 50.97 & 1.53 & -15.23\\
% \bottomrule
% \end{tabular}
% }
% \end{table}
% We see that by removing \textbf{masking} collapses the accuracy and makes RI strongly negative because later updates overwrite parameters that earlier clients depend on. On the other hand, removing \textbf{replay} retains stability but eliminates backward transfer, in line with its design purpose. These observations confirm that masking and replay address our objectives: stability and backward knowledge flow, respectively.

\section{Sensitivity Analysis}

\textbf{Hyper-parameters tuning.} We analyze two key factors in our server-side replay: (i) the weights of the DeepInversion objective (\(\beta_{\mathrm{feat}}, \beta_{\mathrm{TV}}, \beta_{L_2}\)), and (ii) the per-step synthetic data budget. Unless stated otherwise, results are on \textbf{CIFAR-10} with \textbf{5 onboarding steps}. We report PA for new clients, RI for existing clients, and Fréchet Inception Distance (FID) \cite{heusel2017gans}. Higher PA/RI are better, lower FID is better.

\begin{table}[h]
\centering
\caption{Sensitivity to replay loss weights on CIFAR-10 (5 onboarding steps).
Means $\pm$ std over 5 seeds. PA/RI $\uparrow$ higher is better, FID $\downarrow$ lower is better.
Replay time is the average server cost per batch on an RTX A5000.}
\label{tab:sensitivity_weights}
\resizebox{\linewidth}{!}{
\begin{tabular}{cccccc}
\toprule
$\beta_{\mathrm{feat}}$ & $\beta_{\mathrm{TV}}$ & $\beta_{L_2}$ &
PA (\%) $\uparrow$ & RI (\%) $\uparrow$ & FID $\downarrow$ / Time (s) \\
\midrule
0.25 & 0.20 & $1\times 10^{-6}$ & $7.95 \pm 0.37$  & $-0.35 \pm 0.29$ & $65.10$ / $6.20$ \\
0.50 & 0.20 & $1\times 10^{-5}$ & $9.85 \pm 0.41$  & $0.62 \pm 0.33$  & $53.40$ / $7.10$ \\
\textbf{1.00} & \textbf{0.20} & \textbf{$1\times 10^{-5}$} & $\mathbf{11.32 \pm 0.28}$ & $\mathbf{2.08 \pm 0.25}$ & \textbf{47.20} / \textbf{8.00} \\
2.00 & 0.20 & $1\times 10^{-5}$ & $11.45 \pm 0.35$ & $1.96 \pm 0.31$  & $44.80$ / $9.70$ \\
1.00 & 0.50 & $1\times 10^{-5}$ & $10.05 \pm 0.42$ & $1.42 \pm 0.39$  & $50.70$ / $8.50$ \\
1.00 & 0.20 & $5\times 10^{-5}$ & $10.65 \pm 0.33$ & $1.74 \pm 0.27$  & $48.30$ / $8.10$ \\
\bottomrule
\end{tabular}}
\end{table}

\begin{table}[h]
\centering
\caption{Sensitivity to synthetic replay budget per step (DeepInversion).
Means $\pm$ std over 5 seeds on CIFAR-10 (5 steps).
Server time is the average per-batch replay time on an RTX A5000.}
\label{tab:sensitivity_budget}
\resizebox{\linewidth}{!}{
\begin{tabular}{cccccc}
\toprule
Images/step & Iter./image & PA (\%) $\uparrow$ & RI (\%) $\uparrow$ & FID $\downarrow$ & Server time (s) \\
\midrule
32  & 20 & $9.82 \pm 0.41$  & $0.64 \pm 0.27$  & $59.40$ & $1.70$ \\
64  & 20 & $10.55 \pm 0.35$ & $1.05 \pm 0.33$  & $54.10$ & $3.20$ \\
128 & 20 & $11.05 \pm 0.29$ & $1.54 \pm 0.37$  & $49.60$ & $5.10$ \\
\textbf{256} & \textbf{20} & $\mathbf{11.42 \pm 0.32}$ & $\mathbf{2.12 \pm 0.28}$ & \textbf{47.30} & \textbf{8.00} \\
512 & 20 & $11.53 \pm 0.44$ & $2.25 \pm 0.39$  & $46.80$ & $14.20$ \\
\bottomrule
\end{tabular}}
\end{table}

We see that when increasing \(\beta_{\mathrm{feat}}\) tightens alignment between synthetic activations and BatchNorm statistics from client training, shrinking the synthetic–real gap. As \(\beta_{\mathrm{feat}}\) rises, FID drops substantially and RI turns positive, indicating more faithful replay that better routes knowledge to earlier clients. In contrast, \(\beta_{\mathrm{TV}}\) suppresses high-frequency artifacts. A moderate value stabilizes inversion without erasing class-discriminative detail. Pushing it to $0.50$ oversmooths images, slightly hurting PA/RI and increasing solve time. Finally, \(\beta_{L_2}\) prevents degenerate, high-energy solutions. Too small can yield noisy textures, too large over-regularizes and mildly reduces PA/RI even if FID stays low, because the replay set becomes less informative.

\textbf{Synthetic budget.} More images per step reduce FID and improve PA/RI, but server time also increases. We therefore use \textbf{256 images/step, 20 iters/image} as a good point for the main experiments.

\section{Theoretical Analysis}
\label{sec:theory}
Let's $\mathcal{B}_{t}$ be the $t$‑th onboarding batch, $H(e_{c};\phi)$ the global hypernetwork, $m_{c} \in \{0,1\}^{|\theta|}$ the binary mask for client $c$, and $\hat\theta_{c}=m_{c}\odot H(e_{c};\phi)$ the personalized parameters served to that client. We first present formal assumptions and then the main convergence theorem with a detailed proof.

\begin{assumption}[Smoothness]
\label{assumption:smoothness}
The global objective function is defined as 
\[
\mathcal{F}(\phi) = \mathbb{E}_{c}\left[\mathcal{L}_{c}(H(e_c;\phi)\odot m_c)\right]
\]
is $L$-smooth, meaning that there exists $L>0$ such that for all $\phi,\phi'$:
\[
\|\nabla \mathcal{F}(\phi)-\nabla \mathcal{F}(\phi')\| \leq L\|\phi-\phi'\|.
\]
\end{assumption}

\begin{assumption}[Unbiased Gradient]
\label{assumption:unbiased}
For each communication round $r$, the stochastic gradient obtained from participating clients is unbiased:
\[
\mathbb{E}_{c\sim\mathcal{B}_{r}}\left[\nabla_{\phi}\mathcal{L}_{c}(H(e_c;\phi^{(r)})\odot m_c)\right] = \nabla_{\phi}\mathcal{F}(\phi^{(r)}).
\]
\end{assumption}

\begin{assumption}[Bounded Gradient Variance]
\label{assumption:bounded_variance}
There exists a finite constant $\sigma^2 > 0$ such that for every communication round $r$:
\[
\mathbb{E}_{c\sim\mathcal{B}_{r}}\left[\|\nabla_{\phi}\mathcal{L}_{c}(H(e_c;\phi^{(r)})\odot m_c)-\nabla_{\phi}\mathcal{F}(\phi^{(r)})\|^2\right] \leq \sigma^2.
\]
\end{assumption}

\begin{assumption}[Uniform Client Sampling]
\label{assumption:uniform_sampling}
At each communication round $r$, exactly $K$ clients are uniformly sampled without replacement.
\end{assumption}

\begin{assumption}[Mask Immutability]
\label{assumption:mask_immutability}
Each client's mask is updated only within its onboarding batch. Specifically, if client $c$ joins in batch $t_0$, there exists a round $r_{\text{freeze}}(c)$ after which the mask remains fixed:
\[
m_c^{(r)} = m_c^{(r_{\text{freeze}}(c))},\quad \forall r > r_{\text{freeze}}(c).
\]
\end{assumption}

\begin{assumption}[Diminishing Step-Sizes]
\label{assumption:step_size}
The sequence of server learning rates $\{\eta^{(r)}\}$ satisfies Robbins-Monro conditions:
\[
\sum_{r=0}^{\infty}\eta^{(r)} = \infty,\quad \text{and}\quad \sum_{r=0}^{\infty}\left(\eta^{(r)}\right)^2 < \infty.
\]
\end{assumption}

\noindent Now we present the main theorem, as follows:

\begin{theorem}[Convergence to Stationary Point]
\label{theorem:convergence_stationary}
Under Assumptions \ref{assumption:smoothness}-\ref{assumption:step_size}, the sequence of hypernetwork parameters $\{\phi^{(r)}\}$ generated by pFedDSH converges to a stationary point of the global objective $\mathcal{F}$. Formally, we have:
\[
\lim_{r\rightarrow\infty}\mathbb{E}\left[\|\nabla\mathcal{F}(\phi^{(r)})\|^2\right] = 0.
\]
\end{theorem}
\begin{proof}
Consider the iterative update rule at each communication round $r$:
\[
\phi^{(r+1)} = \phi^{(r)} - \eta^{(r)} g^{(r)},
\]
where
\[
g^{(r)}=\frac{1}{K}\sum_{c\in\mathcal{B}_r}\nabla_{\phi}\mathcal{L}_{c}(H(e_c;\phi^{(r)})\odot m_c).
\]
By Assumption \ref{assumption:smoothness} (smoothness), we have:
\[
\begin{aligned}
\mathcal{F}(\phi^{(r+1)}) & \leq \mathcal{F}(\phi^{(r)}) + \langle\nabla\mathcal{F}(\phi^{(r)}),\phi^{(r+1)}-\phi^{(r)}\rangle \\
& + \frac{L}{2}\|\phi^{(r+1)}-\phi^{(r)}\|^2.
\end{aligned}
\]
Substitute the update rule, and taking conditional expectation, by Assumptions \ref{assumption:unbiased} and \ref{assumption:bounded_variance}, we have:
\[
\begin{aligned}
\mathbb{E}\left[\mathcal{F}(\phi^{(r+1)})|\phi^{(r)}\right] & \leq \mathcal{F}(\phi^{(r)}) - \eta^{(r)}\|\nabla\mathcal{F}(\phi^{(r)})\|^2 \\
& + \frac{L}{2}(\eta^{(r)})^2(\sigma^2+\|\nabla\mathcal{F}(\phi^{(r)})\|^2).
\end{aligned}
\]
Taking expectations and summing from $r=0$ to $T$:
\[
\begin{aligned}
\sum_{r=0}^{T}\eta^{(r)}\mathbb{E}\|\nabla\mathcal{F}(\phi^{(r)})\|^2 & \leq \mathcal{F}(\phi^{(0)}) - \mathbb{E}[\mathcal{F}(\phi^{(T+1)})]\\
&+ \frac{L\sigma^2}{2}\sum_{r=0}^{T}(\eta^{(r)})^2.
\end{aligned}
\]
Since $\mathcal{F}(\phi)$ is bounded below, and by Assumption \ref{assumption:step_size}, we have $\sum_{r=0}^{\infty}(\eta^{(r)})^2 < \infty$. Thus, as $T\to\infty$:
\[
\sum_{r=0}^{\infty}\eta^{(r)}\mathbb{E}\|\nabla\mathcal{F}(\phi^{(r)})\|^2 < \infty.
\]
By Robbins-Monro lemma since $\sum_{r}\eta^{(r)}=\infty$, it follows:
\[
\lim_{r\to\infty}\mathbb{E}\|\nabla\mathcal{F}(\phi^{(r)})\|^2 = 0.
\]
Finally, by Assumption \ref{assumption:mask_immutability}, masks remain fixed after the associated batch update. Hence, smoothness and bounded variance conditions continue to hold, validating the previous analysis. Therefore, the sequence converges.
\end{proof}

\section{Neuron-Reuse Guarantees}

In this section, we will show why batch-specific masking does not exhaust neurons even as client batches grow to hundreds, with one theorem that ties capacity growth directly to data novelty, not to the mere count of batches.

\textbf{Setup.}
At onboarding step $t$ for layer $\ell$ of width $C_\ell$, let $m^{(\ell)}_{\le t}\in[0,1]^{C_\ell}$ be the cumulative channel mask and $U^{(\ell)}_t=\{i:m^{(\ell)}_{\le t}(i)=1\}$ the allocated channels with size $A^{(\ell)}_t=|U^{(\ell)}_t|$. Gradients are gated by
\begin{equation}
\begin{aligned}
\frac{\partial \mathcal{L}}{\partial \theta^{(\ell)}_{ij}}
= \frac{\partial \mathcal{L}}{\partial \theta^{(\ell)}_{ij}}\cdot
&\max \Big\{ g^{(\ell)}_t(i)\big(1-m^{(\ell)}_{<t}(i)\big), \\
&g^{(\ell-1)}_t(j)\big(1-m^{(\ell-1)}_{<t}(j)\big)\Big\},
\end{aligned}
\end{equation}
so if both endpoints were previously allocated, the multiplicative factor is $0$ and that connection is frozen. Let $\mathcal{S}^{(\ell)}_{t-1}$ denote the subspace of layer-$\ell$ function-gradients reachable by modifying only weights incident to $U^{(\ell)}_{t-1}$ (the “already-allocated subnetwork”). Let $G^{(\ell)}_t$ be the gradient subspace demanded by the new batch $B_t$ at layer $\ell$, and define the \emph{novelty dimension}
\[
r^{(\ell)}_t  :=  \dim \Big(G^{(\ell)}_t \cap \big(\mathcal{S}^{(\ell)}_{t-1}\big)^\perp\Big).
\]

\begin{theorem}\label{thm:novelty}
Under the gradient gating above, the minimal number of \emph{new} channels that step $t$ can force at layer $\ell$ equals the novelty dimension $r^{(\ell)}_t$. Consequently,
\[
\Delta A^{(\ell)}_t  \le  r^{(\ell)}_t
\quad\text{and}\quad
A^{(\ell)}_T  \le  A^{(\ell)}_0 + \sum_{t=1}^{T} r^{(\ell)}_t,
\]
so layer $\ell$ cannot be exhausted by step $T$ whenever $\sum_{t=1}^{T} r^{(\ell)}_t < C_\ell - A^{(\ell)}_0$.
\end{theorem}

\begin{proof}
Because connections whose endpoints lie in $U^{(\ell)}_{t-1}$ and $U^{(\ell-1)}_{t-1}$ receive zero gradient, the already-allocated subnetwork realizes exactly the directions in $\mathcal{S}^{(\ell)}_{t-1}$. Any component of $G^{(\ell)}_t$ orthogonal to this subspace cannot be effected without activating new channels. Each new channel introduces at most one independent direction for the layer's update, thus at least $r^{(\ell)}_t$ activations are necessary.

Conversely, allocate $r^{(\ell)}_t$ fresh channels and assign their incident weights to span a basis of $G^{(\ell)}_t\cap(\mathcal{S}^{(\ell)}_{t-1})^\perp$. Together with the frozen subspace $\mathcal{S}^{(\ell)}_{t-1}$, this reproduces the full step-$t$ gradient. Hence the minimal new allocation equals $r^{(\ell)}_t$, yielding $\Delta A^{(\ell)}_t\le r^{(\ell)}_t$. Summing over $t$ gives the stated budget inequality; the non-exhaustion condition is the contrapositive of $A^{(\ell)}_T<C_\ell$.
\end{proof}

\end{document}